\documentclass[12pt,a4paper, oneside]{amsart}

\usepackage{amsmath, nicefrac, amsthm, verbatim, amsfonts, mathtools, amssymb, upgreek, xcolor, bbm}
\usepackage{graphics, xspace, enumerate}
\usepackage{stix}
\usepackage{dsfont}
\usepackage[a4paper,margin=2.5cm]{geometry}
\usepackage[bb=boondox]{mathalfa}

\usepackage{graphicx}
\usepackage[colorlinks=true,citecolor=red,urlcolor=blue,linkcolor=red,bookmarksopen=true,unicode=true,pdffitwindow=true]{hyperref}
\usepackage[english]{babel}
\usepackage[languagenames,fixlanguage]{babelbib}
\hypersetup{pdfauthor={}}
\hypersetup{pdftitle={}}

\usepackage{seqsplit,cleveref}

\theoremstyle{plain}
\newtheorem{theorem}{Theorem}[section]

\newtheorem{lemma}[theorem]{Lemma}

\theoremstyle{definition}
\newtheorem{remark}[theorem]{Remark}
\newtheorem{example}[theorem]{Example}
\newtheorem{definition}[theorem]{Definition}

\newcommand {\Prob} {\ensuremath{\mathbb{P}}}
\newcommand {\R} {\ensuremath{\mathbb{R}}}

\newcommand {\N} {\ensuremath{\mathbb{N}}}

\newcommand{\df}{\coloneqq}

\newcommand{\E}{\mathrm{e}}

\newcommand{\sfZ}{\mathsf{Z}}

\newcommand{\X}{\mathrm{X}}

\newcommand{\D}{\mathrm{d}}

\numberwithin{equation}{section}

\title[Rademacher learning rates for iterated random functions]{Rademacher learning rates for iterated random functions}

\author[N.\ Sandri\'{c}]{Nikola Sandri\'{c}}
\address[Nikola\ Sandri\'{c}]{Department of Mathematics\\University of Zagreb\\ Zagreb\\Croatia}
\email{nikola.sandric@math.hr}

\makeatletter
\@namedef{subjclassname@2020}{%
	\textup{2020} Mathematics Subject Classification}
\makeatother

\subjclass[2020]{68W40, 68T10, 60J05}
\keywords{Approximate empirical risk minimization  algorithm, Iterated random function,  Rademacher complexity, Wasserstein distance}

\begin{document}
\allowdisplaybreaks[4]

\begin{abstract}
Most existing literature on supervised machine learning assumes that the training dataset is drawn from an i.i.d. sample. However, many real-world problems exhibit temporal dependence and strong correlations between the marginal distributions of the data-generating process, suggesting that the i.i.d. assumption is often unrealistic. In such cases, models naturally include time-series processes with mixing properties, as well as irreducible and aperiodic ergodic Markov chains.
Moreover, the learning rates typically obtained in these settings are independent of the data distribution, which can lead to restrictive choices of hypothesis classes and suboptimal sample complexities for the learning algorithm. In this article, we consider the case where the training dataset is generated by an iterated random function (i.e., an iteratively defined time-homogeneous Markov chain) that is not necessarily irreducible or aperiodic. Under the assumption that the governing function is contractive with respect to its first argument and subject to certain regularity conditions on the hypothesis class, we first establish a uniform convergence result for the corresponding sample error. We then demonstrate the learnability of the approximate empirical risk minimization algorithm and derive its learning rate bound. Both rates are data-distribution dependent, expressed in terms of the Rademacher complexities of the underlying hypothesis class, allowing them to more accurately reflect the properties of the data-generating distribution.
	
\end{abstract}

\maketitle


\section{Introduction}\label{S1}

Let  $(\mathsf{X},\mathcal{X})$ and $(\mathsf{Y},\mathcal{Y})$ be measurable spaces, and let $(x_1,y_1),\dots,(x_n,y_n)\in\mathsf{X}\times\mathsf{Y}$ represent a training  data set  drawn from random elements $(X_1,Y_1),\dots,(X_n,Y_n):\Omega\to\mathsf{X}\times\mathsf{Y}$, which are defined on a probability space $(\Omega,\mathcal{F},\Prob)$. In general,  the probability measure $\Prob$, or the distribution of the data-generating process $(X_1,Y_1),\dots,(X_n,Y_n)$, is  not known. 
The primary goal of supervised machine learning is to construct a learning algorithm $\mathcal{A}:\bigcup_{n=1}^\infty(\mathsf{X}\times\mathsf{Y})^n\to\mathscr{H}$ that accurately  predicts the functional relationship  between the input (first coordinate) and output (second coordinate) observable, given a family of measurable functions $\mathscr{H}\subseteq\mathsf{Y}^{\mathsf{X}}$, known as the hypothesis class.  Even when  an exact functional relationship exists, it is  typically unknown and may not necessarily belong to  $\mathscr{H}$. More formally, given  a  measurable loss function $\mathcal{L}:\mathsf{Y}\times\mathsf{Y}\to[0,\infty)$, a probability measure $\Prob$,  and parameters $\varepsilon,\delta\in(0,1)$, the objective is to construct an algorithm  $\mathcal{A}$ and determine $n(\Prob,\varepsilon,\delta)\in\N$ such that
$$\Prob\bigl(|\mathrm{er}_\Prob(\mathcal{A}((X_1,Y_1),\dots,(X_n,Y_n))-\inf_{h\in\mathscr{H}}\mathrm{er}_\Prob(h)|<\varepsilon\bigr)\ge 1-\delta\qquad \forall n\ge n(\Prob,\varepsilon,\delta),$$ where $\mathrm{er}_\Prob(h)\df\mathbb{E}_\Prob[\mathcal{L}(h(X),Y)]$ is the expected loss of  hypothesis $h\in\mathscr{H}$ under  $\Prob$. 
The smallest $n(\Prob,\varepsilon,\delta)$  satisfying this condition is  known as the sample complexity of $\mathcal{A}$.
A common approach to this problem is Probably Approximately Correct (PAC) learnability,  which assumes that the sample complexity of a learning algorithm does not depend on the distribution $\Prob$ of the data (see \cite{Anthony-Bartlett-Book-1999}, \cite{Mohri-Rostamizadeh-Talwalkar-Book-2018}, and \cite{Shalev-Shwartz-Ben-David-Book-2014}). However, this assumption is often difficult to justify. By treating all possible data distributions equally, it can lead to restrictive choices of hypothesis classes or suboptimal sample complexity.
 An alternative perspective, which this article focuses on, is based on universal consistency and learning rates, where learning rates explicitly depend on the data distribution (see  \cite{Mohri-Rostamizadeh-Talwalkar-Book-2018} and \cite{Shalev-Shwartz-Ben-David-Book-2014}).
A classical assumption in both approaches is that the elements $(X_1,Y_1),\dots,(X_n,Y_n)$ are i.i.d., as discussed in standard references such as  \cite{Anthony-Bartlett-Book-1999}, \cite{Mohri-Rostamizadeh-Talwalkar-Book-2018}, and \cite{Shalev-Shwartz-Ben-David-Book-2014}. 
However, in many real-world supervised learning tasks, the data exhibit temporal dependence and strong correlations between observations, making the i.i.d. assumption unrealistic.

This work examines the case where the training dataset is drawn from an iterated random function (see below for the precise definition of this process), establishing the   learnability of the corresponding  approximate empirical risk minimization algorithm with data-distribution dependent
sample  complexity, expressed in terms of the Rademacher complexities of $\mathscr{H}$. A common application  of this processes is in generating and modeling image data (see \Cref{EX2} and \cite{Diaconis-Freedman-1999} for more details). 
Consequently, they can be effectively utilized in supervised machine learning tasks related to visualization, such as text classification, image recognition, and object detection.
Formally, let $\mathsf{Z}\subseteq\mathsf{X}\times\mathsf{Y}$ be a   
 complete and separable  metric space with bounded metric $\mathrm{d}_\mathsf{Z}$. Without loss of generality,  assume  $\mathrm{d}_\mathsf{Z}$ is bounded by $1$; otherwise it can be normalized as $\mathrm{d}_\mathsf{Z}/\sup_{z,\bar z\in\sfZ}\mathrm{d}_\mathsf{Z}(z,\bar z)$. 
This assumption is not restrictive, as many machine learning problems naturally operate within bounded domains.
  The space $\mathsf{Z}$ is equipped with its Borel $\sigma$-algebra $\mathfrak{B}(\mathsf{Z})$, generated by $\mathrm{d}_\mathsf{Z}$. Additionally, the projection mappings $\mathrm{pr}_\mathsf{X}:\mathsf{Z}\to\mathsf{X}$ and $\mathrm{pr}_\mathsf{Y}:\mathsf{Z}\to\mathsf{Y}$ 
satisfy the conditions $\mathrm{pr}^{-1}_\mathsf{X}(\mathrm{pr}_\mathsf{X}(\mathsf{Z})\cap\mathcal{X})\subseteq\mathfrak{B}(\mathsf{Z})$ and $\mathrm{pr}^{-1}_\mathsf{Y}(\mathrm{pr}_\mathsf{Y}(\mathsf{Z})\cap\mathcal{Y})\subseteq\mathfrak{B}(\mathsf{Z})$, ensuring the measurability of $\mathrm{pr}_\mathsf{X}$ and $\mathrm{pr}_\mathsf{Y}$. 
 \begin{definition}[Iterated random function]
	Let $Z_0$ be a random element taking values in   $\mathsf{Z}$, and let $\{\vartheta_n\}_{n\ge1}$ be a sequence of i.i.d. random elements, independent of $Z_0$,  taking values in a measurable space $\mathrm{\Theta}$. Consider a measurable function $F:\mathsf{Z}\times\mathrm{\Theta}\to\mathsf{Z}$, where $\mathsf{Z}\times\mathrm{\Theta}$  is endowed with the product $\sigma$-algebra. Both $Z_0$ and $\{\vartheta_n\}_{n\ge1}$ are defined on a probability space $(\Omega,\mathcal{F},\mathbb{P})$.
	For $n\ge1$, iteratively define $Z_n\df F(Z_{n-1},\vartheta_n).$ The resulting process $\{Z_n\}_{n\ge0}$ forms a time-homogeneous Markov chain on $\mathsf{Z}$, known as the iterated random function.
\end{definition}
\noindent 
 The process $\{Z_n\}_{n\ge0}= \{(\mathrm{pr}_\mathsf{X}(Z_n),\mathrm{pr}_\mathsf{Y}(Z_n))\}_{n\ge0}$ represents input data (the first component) paired with their corresponding outputs (the second component). The primary objective is to learn a function from a given hypothesis class $\mathscr{H}$ that, given a training data set  $z_0,\dots,z_{n-1}\in\mathsf{Z}$ drawn from the first $n$ samples of $\{Z_n\}_{n\ge0}$, best approximates the relationship  between the input and output, where the sample complexity is data-distribution dependent and expressed in terms of the Rademacher complexities of $\mathscr{H}$.
To the best of our knowledge,  \cite{Bertail-Portier-2019} and this work  provide the first  data-distribution dependent sample complexity bounds   in a  Markov chain setting. Existing studies in the non-i.i.d. framework typically consider classes of mixing time-series processes or irreducible and aperiodic ergodic Markov chains (such as in \cite{Bertail-Portier-2019}; see the literature review below). In contrast, our approach does not require the irreducibility or aperiodicity of $\{Z_n\}_{n\ge0}$.
A Markov chain  $\{Z_n\}_{n\ge0}$ is irreducible if there exists a $\sigma$-finite measure $\upvarphi(\D z)$ on
$\mathfrak{B}(\mathsf{Z})$ such that for any measurable set $B$ with $\upvarphi(B)>0$, we have
$\sum_{n=0}^\infty \Prob^z(Z_n\in B)>0$ for all $z\in\sfZ$. It is aperiodic if no partition $\{B_1,\dots,B_k\}\subseteq\mathfrak{B}(\mathsf{Z})$ with $k\ge2$ exists such that  $\Prob^z(Z_1\in B_{i+1})=1$ for all $z\in B_i$ and $1\le i \le k-1$, and $\Prob^z(Z_1\in B_{1})=1$ for all $z\in B_k$. A typical example of such process is as follows. Let $\{X_n\}_{n\ge0}$ be an irreducible and aperiodic Markov chain on $\mathsf{X}$, and let $h_0:\mathsf{X}\to\mathsf{Y}$ be measurable. According to \cite[Lemma 3.1]{Sandric-Sebek-2023}, the process $Z_n=(X_n,h_0(X_n))$, $n\ge0$, forms an irreducible and aperiodic Markov chain on $\mathsf{Z}=\{(x,h_0(x)):x\in\mathsf{X}\}$. Here, the first component represents the system state, such as  a vector of cepstral coefficients in speech recognition, the position and velocity of a moving object's  center of gravity in object tracking, or the category of a unit-time price change in market prediction. The second component corresponds to the label associated with each state, such as the emotional state of a speaker, the temporal distance of a tracked object from a reference point, or trading activity (buy/sell/wait).
Irreducibility and aperiodicity ensure that the system can transition between states with positive probability and does not exhibit cyclic behavior over finite time steps. However, in certain scenarios, these assumptions may be unrealistic, as some states might not be reachable from all others. This motivates the study of data-generating processes (Markov chains) that do not necessarily possess these properties.
The labeling function $h_0(x)$ is typically unknown. Given a training dataset and a hypothesis class $\mathscr{H}$ (which does not necessarily contain $h_0(x)$), the goal is to construct a learning algorithm that selects a hypothesis $h\in\mathscr{H}$ that best approximates the true labeling function.


\section{Main results} \label{S2}

Before presenting the main results, we introduce the notation used throughout the article. For $z\in\mathsf{Z}$, let $\mathbb{P}^z(\cdot)$ denotes $\mathbb{P}(\cdot|Z_0=x)$. Given a probability measure $\upmu(\D z)$ on $\mathsf{Z}$ (representing the distribution of $Z_0$), define $\mathbb{P}^\upmu(\cdot)\df\int_\mathsf{Z}\mathbb{P}^z(\cdot)\upmu(\D z)$. For $n\ge1$, the $n$-step transition functions of $\{Z_n\}_{n\ge0}$ are given by
$\mathcal{P}^n(z,\D \bar z)\df\mathbb{P}^{z}(Z_n\in \D \bar z)$ when starting from $z$,  and $\upmu\mathcal{P}^n(\D z)\df\mathbb{P}^\upmu(Z_n\in\D z)$ when the  initial distribution is  $\upmu(\D z)$. 
We impose the following assumption (recall that $\{\vartheta_n\}_{n\ge1}$ is defined on the probability space $(\Omega,\mathcal{F},\Prob)$):

\medskip

\begin{description}
	\item[(A1)] For all $z,\bar z\in\mathsf{Z}$ and $\theta\in\Theta$, there exists a measurable function  $\ell:\Theta\to[0,\infty)$ such that $$\mathrm{d}_\mathsf{Z}(F(z,\theta),F(\bar z,\theta))\le \ell(\theta)\mathrm{d}_\mathsf{Z}(z,\bar z)\qquad\text{and}\qquad\ell_F\df\mathbb{E}[\ell(\vartheta_1)]<1.$$ 
	
	

\end{description}

\medskip

\noindent In assumption (A1), $\ell(\theta)$ is defined as the smallest $\ell\ge0$ such that $$\mathrm{d}_\mathsf{Z}(F(z,\theta),F(\bar z,\theta))\le \ell\mathrm{d}_\mathsf{Z}(z,\bar z)\qquad \forall z,\bar z\in\mathsf{Z}.$$ According to \cite[Lemma 5.1]{Diaconis-Freedman-1999}, the mapping $\theta\mapsto\ell(\theta)$ is  measurable. Hence, $\ell_F$ is well defined. Notably,   when $\mathrm{\Theta}=\sfZ$ and  $F(z,\theta)=\theta$ for all $(z,\theta)\in\sfZ\times\mathrm{\Theta}$, assumption (A1) holds trivially.  This demonstrates that the class of data-generating processes considered in this article generalizes the classical i.i.d. case.
Let $\mathscr{P}(\sfZ)$ denote the set of all probability measures on $\sfZ$. The
$\mathrm{L}^1$-Wasserstein distance on $\mathscr{P}(\sfZ)$  is given by
\begin{equation*}
\mathscr{W}(\upmu_1,\upmu_2)\df\inf_{\Pi\in\mathcal{C}(\upmu_1,\upmu_2)}
\int_{\sfZ\times\sfZ}\mathrm{d}_\mathsf{Z}(z,\bar z)
\Pi(\D{z},\D{\bar z}),
\end{equation*}
where $\mathcal{C}(\upmu_1,\upmu_2)$ is the set of all couplings  of
$\upmu_1(\D z)$ and $\upmu_2(\D \bar z)$,
meaning that $\Pi\in\mathcal{C}(\upmu_1,\upmu_2)$ 
is a probability
measure on $\sfZ\times\sfZ$ with marginals $\upmu_1(\D z)$ and $\upmu_2(\D \bar z)$. By the Kantorovich-Rubinstein theorem, 
$$
\mathscr{W}(\upmu_1,\upmu_2) = \sup_{\{f\colon\mathrm{Lip}(f)\le1\}}\,
| \upmu_1(f)-\upmu_2(f)|,
$$
where the supremum is taken over all Lipschitz  functions
$f\colon\sfZ\to\R$ with Lipschitz constant $\mathrm{Lip}(f)\le1$, defined as the smallest $L\ge0$ for which $$|f(z)-f(\bar z)|\le L\mathrm{d}_\mathsf{Z}(z,\bar z)\qquad \forall z,\bar z\in\sfZ.$$  For $\upmu\in \mathscr{P}(\sfZ)$   and a measurable function $f:\sfZ\to\R$, the notation $\upmu(f)$ represents the integral $\int_\sfZ f(z)\upmu(\D z)$, whenever  well defined. It is well known that $(\mathscr{P}(\sfZ),\mathscr{W})$ is a complete separable metric space (see, e.g., \cite[Theorem 6.18]{Villani-Book-2009}). From assumption (A1), it  follows that 
$$\mathscr{W}(\mathcal{P}(z,\D w),\mathcal{P}(\bar z,\D w))\le \mathbb{E}[
	\mathrm{d}_\mathsf{Z}(F(z,\vartheta_1),F(\bar z,\vartheta_1))]\le \ell_F\mathrm{d}_\mathsf{Z}(z,\bar z)\qquad \forall z,\bar z\in\sfZ.$$ In particular,  for any Lipschitz  function $f\colon\sfZ\to\R$,
	$$|\mathcal{P}(f)(z)-\mathcal{P}(f)(\bar z)|\le \mathrm{Lip}(f)\ell_F\mathrm{d}_\mathsf{Z}(z,\bar z)\qquad \forall z,\bar z\in\sfZ,$$
which implies that  $z\mapsto\mathcal{P}(f)(z)$	is Lipschitz  with Lipschitz constant at most $\mathrm{Lip}(f)\ell_F$. Consequently, for any  $\upmu_1,\upmu_2\in\mathscr{P}(\sfZ)$,
	 $$\mathscr{W}(\upmu_1\mathcal{P},\upmu_2\mathcal{P}) \le \ell_F\mathscr{W}(\upmu_1,\upmu_2).$$ Since $\ell_F<1$, the mapping $\upmu\mapsto\upmu\mathcal{P}$ is a contraction on
$\mathscr{P}(\sfZ)$. By the Banach fixed point theorem,  there exists a unique $\uppi\in\mathscr{P}(\sfZ)$ satisfying $\uppi\mathcal{P}(\D z)=\uppi(\D z)$, meaning that $\uppi(\D z)$ is the unique invariant probability measure of $\{Z_n\}_{n\ge0}$. Moreover, for any $n\ge1$ and  $\upmu\in\mathscr{P}(\sfZ)$,  $$\mathscr{W}(\upmu\mathcal{P}^n,\uppi)=\mathscr{W}(\upmu\mathcal{P}^n,\uppi\mathcal{P}^n) \le \ell_F\mathscr{W}(\upmu\mathcal{P}^{n-1},\uppi\mathcal{P}^{n-1})\le\dots\le \ell_F^n\mathscr{W}(\upmu,\uppi).$$

\noindent Further, for   $h\in\mathscr{H}$, define the function $\mathcal{L}_h:\mathsf{Z}\to[0,\infty)$ by $$\mathcal{L}_h(z)\df\mathcal{L}(h(\mathrm{pr}_\mathsf{X}(z)), \mathrm{pr}_\mathsf{Y}(z)).$$ 
We now impose the following assumption on the loss function $\mathcal{L}$ and the hypothesis class $\mathscr{H}$:

\medskip

\begin{description}
	\item[(A2)]  There exists a constant $\ell_\mathscr{H}>0$ such that $$\mathcal{L}_{h}(z)\le\ell_\mathscr{H}\quad\text{and}\quad  |\mathcal{L}_{h}(z)-\mathcal{L}_{h}(\bar z)|\le \ell_\mathscr{H}\mathrm{d}_\mathsf{Z}(z,\bar z)\qquad \forall z,\bar z\in\sfZ,\ h\in\mathscr{H}.$$
\end{description}

\medskip

\noindent 
Examples satisfying assumptions (A1) and (A2) are provided in \Cref{EX1,EX2,EX3}. For $n,m\in\N$ and $h\in\mathscr{H}$, let $\hat{\rm er}_n(h)\df \frac{1}{n}\sum_{i=0}^{n-1}\mathcal{L}_{h}(Z_i)$, $\hat{\rm er}_{n,m}(h)\df \frac{1}{m-n}\sum_{i=n}^{m-1}\mathcal{L}_{h}(Z_i)$  (when $m>n$), ${\rm er}_\uppi(h)\df  \uppi(\mathcal{L}_{h})$ and ${\rm opt}_\uppi(\mathscr{H})\df\inf_{h\in\mathscr{H}}{\rm er}_\uppi(h)$. 
For 
$\varepsilon>0$, the $\varepsilon$-approximate empirical risk minimization ($\varepsilon$-ERM) algorithm for $\mathscr{H}$ is defined  as a mapping $\mathcal{A}^\varepsilon:\bigcup_{n=1}^\infty\sfZ^n\to \mathscr{H}$ satisfying $$\frac{1}{n} \sum_{i=0}^{n-1}\mathcal{L}_{\mathcal{A}^\varepsilon(z_0,\dots,z_{n-1})}(z_i)\le \inf_{h\in\mathscr{H}} \frac{1}{n} \sum_{i=0}^{n-1}\mathcal{L}_{h}(z_i)+\varepsilon.$$
For further details on  $\varepsilon$-ERM, we refer the reader to \cite{Anthony-Bartlett-Book-1999} and the reference therein. Finally,  let $\{\sigma_n\}_{n\ge0}$ be an i.i.d. sequence of symmetric Bernoulli random variables (taking values in  $\{-1,1\}$)  defined on a probability space $(\Omega^\sigma,\mathcal{F}^\sigma,\Prob^\sigma)$,  independent of $\{Z_n\}_{n\ge0}$. The $n$-empirical Rademacher complexity of the function class
$\{\mathcal{L}_h\colon h\in\mathscr{H}\}$ with respect to $z_0,\dots,z_{n-1}\in\sfZ$ is given by $$\hat{\mathcal{R}}_{n,(z_0,\dots,z_{n-1})}(\mathscr{H})\df\mathbb{E}^\sigma\left[\sup_{h\in\mathscr{H}}\frac{1}{n}\sum_{i=0}^{n-1}\sigma_i\mathcal{L}_h(z_i)\right].$$ Similarly,  the $(n,\upmu)$-Rademacher complexity of
$\{\mathcal{L}_h\colon h\in\mathscr{H}\}$ with respect to $\{Z_n\}_{n\ge0}$ with initial distribution  $\upmu(\D z)$,  is defined as $\mathcal{R}_{n,\upmu}(\mathscr{H})\df\mathbb{E}^\upmu[\hat{\mathcal{R}}_{n,(Z_0,\dots,Z_{n-1})}(\mathscr{H})]$.
The Rademacher complexity
measures, on average, how well the function class $\{\mathcal{L}_h\colon h\in\mathscr{H}\}$ correlates with random noise $\{\sigma_n\}_{n\ge0}$ on the given dataset. In general, richer or more complex function classes   tend to have higher Rademacher complexity, as they exhibit stronger correlations with random noise. For further details on Rademacher complexity, see
 \cite{Mohri-Rostamizadeh-Talwalkar-Book-2018} and \cite{Shalev-Shwartz-Ben-David-Book-2014}.

\medskip

We now state the main result of this article.

\begin{theorem}\label{TM} Assume (A1) and (A2). For any $\upmu\in\mathscr{P}(\sfZ)$ and  $\varepsilon\in(0,1)$, we have
\begin{align*}&\Prob^\upmu\left(|\mathrm{er}_\uppi\bigl(\mathcal{A}^\varepsilon(Z_0,\dots,Z_{n-1})\bigr)-\mathrm{opt}_\uppi(\mathscr{H})|<4\mathcal{R}_{n,\uppi}(\mathscr{H})+2\ell_\mathscr{H}\ell_F^n\mathscr{W}(\upmu,\uppi)+4\varepsilon\right)\\&\ge1-2\E^{-2\varepsilon^2n/(\ell_\mathscr{H}/(1-\ell_F))^2}.\end{align*}
Furthermore, for any $\varepsilon\in(0,1)$, it holds that
\begin{align*}&\Prob^\uppi\left(|\mathrm{er}_\uppi\bigl(\mathcal{A}^\varepsilon(Z_0,\dots,Z_{n-1})\bigr)-\mathrm{opt}_\uppi(\mathscr{H})|<4\hat{\mathcal{R}}_{n,(Z_0,\dots,Z_{n-1})}(\mathscr{H})+6\varepsilon\right)\\&\ge1-2\E^{-2\varepsilon^2n/(\ell_\mathscr{H}/(1-\ell_F))^2}.\end{align*}
		\end{theorem}

\noindent The proof of \Cref{TM} follows a standard approach (see, e.g., \cite{Mohri-Rostamizadeh-Talwalkar-Book-2018}). Specifically, leveraging the contractivity  (in the first variable) of the governing function of $\{Z_n\}_{n\ge0}$, along with the Lipschitz continuity  of the loss function and the hypothesis class,  and applying a Hoeffding's-type inequality, we first establish a uniform convergence result for the corresponding sample error in terms of the Rademacher and empirical Rademacher complexities of $\mathscr{H}$ (see \Cref{LEM1,LEM2,LEM3}). This result then allows us to conclude the learnability of the approximate empirical risk minimization algorithm and derive its generalization bounds.  To the best of our knowledge, the only closely related result appears in \cite{Bertail-Portier-2019}, where the authors obtain data-dependent learning rates (expressed in terms of Rademacher complexities) for a class of irreducible and aperiodic ergodic Markov chains that admit an atom. However, unlike \cite{Bertail-Portier-2019}, our work does not assume irreducibility, aperiodicity, or atomic structure of the underlying Markov chain (modeled here as an iterated random function).
As in the classical i.i.d. setting (see \cite{Mohri-Rostamizadeh-Talwalkar-Book-2018}), the learning rate obtained in \Cref{TM} is exponential. Notably, a similar rate (with different constants) has also been established in \cite{Bertail-Portier-2019} and in other works addressing the non-i.i.d. setting, including \cite{Gao-Niu-Zhou-2016}, \cite{Kuznetsov-Mohri-2014}, \cite{Kuznetsov-Mohri-2016}, and \cite{Mohri-Rostamizadeh-2008}.

Next, we recall that if $\mathsf{Y}$ is finite (as in classification or ranking problems), then from \cite[Theorem 3.7]{Mohri-Rostamizadeh-Talwalkar-Book-2018} (with $r:=L_\mathscr{H}\sqrt{n}$), it follows that $$\hat{\mathcal{R}}_{n,(z_0,\dots,z_{n-1})}(\mathscr{H})\le L_\mathscr{H}\sqrt{\frac{2\log {\rm r}_{\{\mathcal{L}_h\colon h\in\mathscr{H}\}}(n)}{n}}.$$ Here, $L_\mathscr{H}\df\sup_{z\in\mathsf{Z},\, h\in\mathscr{H}}\mathcal{L}_h(z)$ and $ {\rm r}_{\{\mathcal{L}_h\colon h\in\mathscr{H}\}}:\N\to\N$ denotes the growth function of the class $\{\mathcal{L}_h\colon h\in\mathscr{H}\}$ (see, e.g., \cite[Section 3.2]{Anthony-Bartlett-Book-1999}. By definition,  $${\rm r}_{\{\mathcal{L}_h\colon h\in\mathscr{H}\}}(n)\le\min\{{\rm card}(\mathscr{H}),{\rm card}(\{\mathcal{L}_h(z)\colon z\in\mathsf{Z},h\in\mathscr{H})^n\}.$$ In particular, for binary classification (i.e., when ${\rm card}(\mathsf{Y})=2$), we obtain 
$$\hat{\mathcal{R}}_{n,(z_0,\dots,z_{n-1})}(\mathscr{H})\le L_\mathscr{H}\sqrt{\frac{2{\rm VC}(\{\mathcal{L}_h\colon h\in\mathscr{H}\})\log(\E n/{\rm VC}(\{\mathcal{L}_h\colon h\in\mathscr{H}\}))}{n}},$$
where ${\rm VC}(\{\mathcal{L}_h\colon h\in\mathscr{H}\})$ denotes the Vapnik-Chervonenkis dimension of  $\{\mathcal{L}_h\colon h\in\mathscr{H}\}$ (see \cite[Corollary 3.8]{Anthony-Bartlett-Book-1999}).
On the other hand, if $\{\mathcal{L}_h\colon h\in\mathscr{H}\}$ consists of bounded functions (so that $\mathsf{Y}$ need not necessarily be finite), then $$\hat{\mathcal{R}}_{n,(z_0,\dots,z_{n-1})}(\mathscr{H})\le\inf_{\alpha\ge0}\left\{4\alpha+\frac{c_1}{\sqrt{n}}\int_\alpha^1\sqrt{{\rm fat}_{\delta}(\{\mathcal{L}_h\colon h\in\mathscr{H}\})\log(c_2/\delta)}\,\D\delta\right\},$$
where the constants $c_1$ and $c_2$ depend only on the boundary points of the interval containing the images of $\{\mathcal{L}_h\colon h\in\mathscr{H}\}$. The term ${\rm fat}_{\delta}(\{\mathcal{L}_h\colon h\in\mathscr{H}\})$ represents the $\delta$-fat-shattering dimension of $\{\mathcal{L}_h\colon h\in\mathscr{H}\}$ (see \cite[Lecture 12]{Sridharan-15}).

Let us now present several examples of iterated random functions that satisfy the conditions of \Cref{TM}. We begin with an example designed to generate image data.

\begin{example} \label{EX2}{\rm   
			Let $d,m\in\N$ and $R>0$,  and let $\mathsf{X}=\R^d$ and  $\mathsf{Y}=\R^m$.
			We equip these spaces with the standard Euclidean norms,  denoted by $\lVert\cdot\lVert_\mathsf{X}$ and $\lVert\cdot\lVert_\mathsf{Y}$, respectively.
				Further, let $\upeta(\D x)$ be a probability measure on $\mathsf{X}$ satisfying $\upeta(\bar{B}_\mathsf{X}(0,R)^c)=0$, where $\bar{B}_\mathsf{X}(x_0,\rho)$ denotes the closed ball of radius $\rho>0$ centered at $x_0\in\mathsf{X}$. The measure $\upeta(\D x)$ represents an image contained within the closed ball of radius $R$ around the origin. According to \cite{Diaconis-Freedman-1999}, for any $k\ge2$, there exist:
			\begin{itemize}
				\item[(i)]affine transformations $(a_1,b_1),\dots,(a_k,b_k)\in\R^{d\times d}\times\R^{d\times1}$, where each $a_i$ is a contraction, i.e., $\lVert a_i\lVert<1$ (where $\lVert \cdot\lVert$ denotes the spectral norm)
				\item[(ii)] a probability measure $\upnu(\D i)$ on $\{1,\dots,k\}$,
			\end{itemize}
such that for an i.i.d. sequence $\{\vartheta_n\}_{n\ge1}=\{(A_n,B_n)\}_{n\ge1}$ on $\mathrm{\Theta}=\{(a_1,b_1),\dots,(a_k,b_k)\}$ with distribution $\upnu(\D i)$, 			
 a random variable  $X_0$  on $\bar{B}_\mathsf{X}(0,R+r)$  independent of $\{\vartheta_n\}_{n\ge1}$, where $r=\max_{1,\dots k}\lVert b_i\lVert_\mathsf{X}$, and for any $f:\bar{B}_\mathsf{X}(0,R+r)\times\mathrm{\Theta}\to\bar{B}_\mathsf{X}(0,R+r)$ satisfying $f(x,(a_i,b_i))=a_ix+b_i$ for $x\in\bar{B}_\mathsf{X}(0,R)$ and $\lVert f(x,(a_i,b_i))-f(\bar x,(a_i,b_i))\lVert_\mathsf{X}\le \lVert a_i\lVert\lVert x-\bar x\lVert_\mathsf{X},$ the sequence 
$$X_{n}:=f(X_{n-1},\vartheta_n),\qquad n\ge1,$$ defines a Markov chain on 
$\bar{B}_\mathsf{X}(0,R+r)$ with invariant probability measure $\upeta(\D x)$. Denote the transition function of $\{X_n\}_{n\ge0}$ by $\mathcal{P}_X(x,\D \bar x)$.  Then,  for any $n\ge1$ and  $\upmu\in\mathscr{P}(\bar{B}_\mathsf{X}(0,R+r))$, it holds that $$\mathscr{W}(\upmu\mathcal{P}_X^n,\upeta)=\mathscr{W}(\upmu\mathcal{P}_X^n,\upeta\mathcal{P}^n_X) \le \mathbb{E}[\lVert \vartheta_1\lVert]\mathscr{W}(\upmu\mathcal{P}_X^{n-1},\upeta\mathcal{P}^{n-1}_X)\le\dots\le \mathbb{E}[\lVert \vartheta_1\lVert]^n\mathscr{W}(\upmu,\upeta).$$ Since each $a_i$ is a contraction, we have $\mathbb{E}[\lVert \vartheta_1\lVert]<1$, implying that $\upeta(\D x)$ is the unique invariant probability measure of $\{X_n\}_{n\ge0}$, and $\upmu\mathcal{P}_X^n(\D x)$ approximates the image given by $\upeta(\D x)$ in the Wasserstein sense. 
Next,  let $\sfZ$ be the closed ball of radius $R+r$ around the origin  in $\mathsf{X}\times		\mathsf{Y}$, equipped with the metric
				$${\rm d}_\mathsf{Z}(z,\bar z):=\frac{\lVert x-\bar x\rVert_\mathsf{X}+\lVert y-\bar y\rVert_\mathsf{Y}}{4(R+r)},\qquad z=(x,y),\bar z=(\bar x,\bar y)\in\sfZ.$$
	Observe that $\mathrm{pr}_\mathsf{X}(\mathsf{Z})=\bar{B}_\mathsf{X}(0,R+r)$ and ${\rm d}_\mathsf{Z}$ is bounded by $1$.
	Further, let 
 $h_0\colon\mathrm{pr}_\mathsf{X}(\mathsf{Z})\to\mathrm{pr}_\mathsf{Y}(\mathsf{Z})$ be Lipschitz and
satisfying $\mathbb{E}[\lVert \vartheta_1\lVert](1+{\rm Lip}(h_0))<4(R+r)$.
 Finally, define $F:\sfZ\times\mathrm{\Theta}\to\sfZ$ by
$$F(z,\theta):=(f(x,\theta),h_0(f(x,\theta))),\qquad z=(x,y)\in\sfZ.$$  
Since, $$\ell(\theta)\le\frac{\rVert\theta\rVert(1+{\rm Lip}(h_0))}{4(R+r)},$$ by assumption, $\ell_F<1$, ensuring that assumption (A1) holds.
Consequently, the process $Z_n:=F(Z_{n-1},\vartheta_n)=(X_n,h_0(X_n))$, $n\ge1$, with $Z_0:=(X_0,h_0(X_0))$, defines a Markov chain on $\sfZ$ with a unique invariant probability measure $\uppi(\D z)$. Assumption  (A2) holds, for instance, if  $\mathcal{L}$ is  Lipschitz,  the hypothesis class $\mathscr{H}$ consists of Lipschitz  functions satisfying $\sup_{h\in \mathscr{H}}{\rm Lip}(h)<\infty,$ and  the function $(z,h)\mapsto\mathcal{L}_h(z)$ is bounded.
	   \qed 
	}
\end{example}

The previous example can be be placed in a more general framework.

\begin{example} \label{EX1}{\rm  	Let $\X$ and $\mathsf{Y}$ be  separable and complete metric spaces  equipped with  metrics
		${\rm d}_\mathsf{X}$ and  
		${\rm d}_\mathsf{Y},$ respectively,
		and let $\sfZ\subset\X\times\mathsf{Y}$ be bounded 
		and set $\kappa:=\sup_{(x,\bar x),(y,\bar y)\in\sfZ}({\rm d}_\mathsf{X}(x,\bar x)+{\rm d}_\mathsf{Y}(y,\bar y))$. Consider the metric
		$${\rm d}_\mathsf{Z}(z,\bar z):=\frac{{\rm d}_\mathsf{X}(x,\bar x)+{\rm d}_\mathsf{Y}(y,\bar y)}{\kappa},\qquad z=(x,y),\bar z=(\bar x,\bar y)\in\sfZ.$$ 
		Next, let $f\colon\mathrm{pr}_\mathsf{X}(\mathsf{Z})\times\mathrm{\Theta}\to \mathrm{pr}_\mathsf{X}(\mathsf{Z})$ be measurable function satisfying 
		$${\rm d}_\mathsf{X}(f(x,\theta),f(\bar x,\theta))_\mathsf{X}\le\ell(\theta){\rm d}_\mathsf{X}(x,\bar x)$$ for some measurable $\ell:\mathrm{\Theta}\to[0,\infty)$ with $\ell_f:=\mathbb{E}[\ell(\vartheta_1)]<\infty.$ Let
		$h_0\colon\mathrm{pr}_\mathsf{X}(\mathsf{Z})\to\mathrm{pr}_\mathsf{Y}(\mathsf{Z})$ be Lipschitz, and define
		$$F(z,\theta):=(f(x,\theta),h_0(f(x,\theta))),\qquad z=(x,y)\in\sfZ.$$ Assume that $\ell_f(1+{\rm Lip}(h_0))<\kappa.$
	It follows directly that $$\ell_F\le  \frac{\ell_f(1+{\rm Lip}(h_0))}{\kappa},$$ which ensures that assumption (A1) is satisfied. 
		As in the previous example, assumption (A2) holds under mild assumptions, such as when  $\mathcal{L}$ is  Lipschitz,  $\mathscr{H}$ consists of Lipschitz  functions satisfying $\sup_{h\in \mathscr{H}}{\rm Lip}(h)<\infty,$ and  the function $(z,h)\mapsto\mathcal{L}_h(z)$ is bounded.\qed
	}
\end{example}

In the following example, we consider an iterated random function that is neither irreducible nor aperiodic.

\begin{example} \label{EX3}{\rm  
	Let $\{Z_n\}_{n\ge0}$ be an iterated random function from \Cref{EX1} in which the function $f(x,\theta)$ does not depend on $\theta$, i.e., $f(x,\theta)=f(x)$  for some  $f\colon\mathrm{pr}_\mathsf{X}(\mathsf{Z})\to \mathrm{pr}_\mathsf{X}(\mathsf{Z})$.  In this case, it is straightforward to verify that the unique invariant probability measure is given by $\uppi(\D z)=\updelta_{z_0}(\D z)$, where $z_0=(x_0,y_0)\in\sfZ$ is the unique solution to $z_0=(f(x_0)),h_0(f(x_0))$. 
		Moreover, one can easily construct concrete examples of such Markov chains (e.g., in $\R^2$) that satisfy the given assumptions but are neither irreducible nor aperiodic. Furthermore, their
		 $n$-step transition functions do not necessarily converge to  $\uppi(\D z)$ in the total variation distance.\qed
	}
\end{example}

\noindent For further examples of iterated random functions satisfying the conditions of \Cref{TM}, we refer the reader to \cite{Diaconis-Freedman-1999}.

\section{Literature review}\label{S3}

Our work contributes to the understanding of  the statistical properties of supervised machine learning problems. Much of the existing literature focuses on PAC learnability under the assumption that the training dataset is drawn from an i.i.d. sample, as discussed in the classical monographs \cite{Anthony-Bartlett-Book-1999}, \cite{Mohri-Rostamizadeh-Talwalkar-Book-2018}  and \cite{Shalev-Shwartz-Ben-David-Book-2014}. However, many practical supervised learning problems—such as speech recognition, object tracking, and market prediction—exhibit temporal dependence and strong correlations within the data-generating process, making the i.i.d. assumption questionable.
The first study addressing PAC learnability in such settings appeared in \cite{Vidyasagar-Book-2003}, while consistency was examined in \cite{Meir-2000}, \cite{Modha-Masry-1998}, and \cite{Yu-1994},  within the framework of stationary mixing time-series processes. Further results in this context can be found in \cite{Bosq-Book-1998} and \cite{Gyorfi-Hardle-Sarda-Vieu-Book-1989}. 
In  \cite{Irle-1997} and \cite{Steinwart-Hush-Scovel-2009}, the authors relaxed the stationarity assumption,  requiring only that the
data-generating process satisfies a specific  law of large numbers in the former work and certain mixing properties in the latter. These approaches encompass (non-)stationary mixing time-series processes as well as irreducible and aperiodic ergodic Markov chains.
The PAC learnability of this model was further explored in a series of works \cite{Gamarnik-2003}, \cite{Zou-Li-Xu-2012}, \cite{Zou-Li-Xu-Luo-Tang-2013}, \cite{Zou-Xu-Xu-2014} and \cite{Zou-Zhang-Xu-2009}. A broader generalization, considering PAC learnability for not necessarily irreducible and aperiodic ergodic Markov chains, was proposed in \cite{Sandric-Sebek-2023}. 

Related results have also been studied in the context of general concentration inequalities. Classical references include \cite{Boucheron-Lugosi-Massart-Book-2013}, \cite{Dubhashi-Panconesi-Book-2009} and \cite{Ledoux-Book-2001}. 
By employing coupling techniques and imposing constraints on the coupling time, analogous concentration inequalities for general coordinate-wise Lipschitz functions evaluated along the sample path of a stationary ergodic Markov chain were established in \cite{Chazottes-Redig-2009} and \cite{Paulin-2015}. For the related works see also the references therein. By forgoing explicit assumptions of stationarity, irreducibility, and aperiodicity, and instead utilizing martingale techniques, similar results have been derived in \cite{Kontorovich-Raginsky-2017} and \cite{Kontorovich-Ramanan-2008}. A key property in these studies is the uniform contractivity of the transition function with respect to the total variation distance, a condition known as the Markov-Dobrushin condition (see, e.g., \cite[Theorem 4.1]{Hairer-Lecture-notes-2016} or \cite[Chapter 3]{Kulik-2018}).  This condition encapsulates stationarity, irreducibility, and aperiodicity of the underlying Markov chain.

Across these studies, the obtained learning rates are either data-distribution-independent (in PAC or concentration inequality results) or remain unknown (in consistency results). Data-distribution-dependent (or consistency) learning rates for i.i.d. samples are provided, for example, in \cite{Mohri-Rostamizadeh-Talwalkar-Book-2018}  and \cite{Shalev-Shwartz-Ben-David-Book-2014}.
 The first results on consistency learning rates in non-i.i.d. settings appeared in \cite{Mohri-Rostamizadeh-2008}, where the authors considered a class of stationary mixing time-series processes, expressing data-distribution dependence in terms of the Rademacher complexities of the underlying hypothesis class. This result was later extended to empirical processes of stationary mixing time-series in \cite{Gao-Niu-Zhou-2016} and to non-stationary mixing time-series processes in \cite{Kuznetsov-Mohri-2014} and \cite{Kuznetsov-Mohri-2016}. The only related result in the context of Markov chains is found in \cite{Bertail-Portier-2019}, where the authors establish bounds on the Rademacher complexities of a class of Vapnik-Chervonenkis type functions and derive consistency learning rates (in terms of Rademacher complexities) for irreducible and aperiodic ergodic Markov chains admitting an atom. In contrast, in this work, we do not assume irreducibility, aperiodicity, or an atomic structure in the underlying Markov chain (iterated random function).

The ergodic properties of iterated random functions—specifically, stationarity and convergence to the corresponding invariant probability measure—were first studied in \cite{Diaconis-Freedman-1999}. Based on the contractivity nature and Lipschitz continuity of the governing function, as assumed in this paper (see conditions (A1) and (A2)), \cite{Dedecker-Fan-2015} derived bounds on various concentration inequalities. These results were further applied to bound the Wasserstein distance between the empirical and invariant distributions of the iterated random function. A key distinction between their work and ours is that \cite{Dedecker-Fan-2015} assumes Lipschitz continuity in the second variable of the governing function—a requirement that excludes important examples, such as the one in \Cref{EX2}. Finally, by replacing the governing function 
$F(z,\theta)$ with a sequence of functions 
$\{F_n(z,\theta)\}_{n\ge0}$ satisfying the same contractivity and Lipschitz conditions (in both variables), the results of \cite{Dedecker-Fan-2015} were extended in \cite{Alquier-Doukhan-Fan-2019} to a class of time-inhomogeneous iterated random functions.


\section{Proof of \Cref{TM}}\label{S4}
In this section, we establish the proof of  \Cref{TM}. We begin by deriving a uniform convergence result for the associated sample error.

\begin{lemma}\label{LEM1} Assume (A1) and (A2). For any $\upmu\in\mathscr{P}(\sfZ)$ and  $\varepsilon\in(0,1)$, the following inequality holds: $$\Prob^\upmu\left(\sup_{h\in\mathscr{H}}|\hat{\rm er}_{n,2n}(h)- {\rm er}_\uppi(h)|\ge \mathbb{E}^\upmu\bigl[\sup_{h\in\mathscr{H}}|\hat{\rm er}_{n,2n}(h)- {\rm er}_\uppi(h)|\bigr]+
	\varepsilon\right)\le \E^{-2\varepsilon^2n/(\ell_\mathscr{H}/(1-\ell_F))^2}.$$
	
\end{lemma}
\begin{proof}
	Fix $n\in\N$, and	define the function $\varphi:\mathsf{Z}^n\to\R$ as $$\varphi(z_1,\dots,z_{n})\df\sup_{h\in\mathscr{H}}\left|\frac{1}{n}\sum_{i=1}^{n}\mathcal{L}_{h}(z_i)- {\rm er}_\uppi(h)\right|.$$ Observe that 
	
$$
	\varphi(z_1,\dots,z_{n})\le\sup_{h\in\mathscr{H}}\frac{1}{n}\sum_{i=1}^{n}\left(\int_{\mathsf{Z}}|\mathcal{L}_{h}(z_i)- \mathcal{L}_{h}(z)|\uppi(\D z)\right)\le\ell_\mathscr{H}.
$$ Additionally, for any $(z_1,\dots,z_n), (\bar{z}_1,\dots,\bar{z}_n) \in \mathsf{Z}^n$,
\begin{equation}\begin{aligned}\label{LIP1}|\varphi(z_1,\dots,z_{n})-\varphi(\bar z_1,\dots,\bar z_{n})|&\le \sup_{h\in\mathscr{H}}\left|\left|\frac{1}{n}\sum_{i=1}^{n}\mathcal{L}_{h}(z_i)- {\rm er}_\uppi(h)\right|-\left|\frac{1}{n}\sum_{i=1}^{n}\mathcal{L}_{h}(\bar z_i)- {\rm er}_\uppi(h)\right|\right|\\
	&\le \sup_{h\in\mathscr{H}}\frac{1}{n}\sum_{i=1}^{n}|\mathcal{L}_{h}(z_i)-\mathcal{L}_{h}(\bar z_i)|.\end{aligned}\end{equation} Moreover, we have $\varphi(Z_n,\dots,Z_{2n-1})=\sup_{h\in\mathscr{H}}|\hat{\rm er}_{n,2n}(h)- {\rm er}_\uppi(h)|$.
	Due to boundedness,   the conditional expectation $\mathbb{E}^\upmu[\varphi(Z_n,\dots,Z_{2n-1})|Z_{n},\dots,Z_i]$ is well defined for all $\upmu\in\mathscr{P}(\sfZ)$ and $i = n, \dots , 2n-1$.
Define the functions $f_{i}:\mathsf{Z}^{i-n+1}\to\R$  by $$f_{i}(Z_n,\dots,Z_i)\df\mathbb{E}^\upmu[\varphi(Z_n,\dots,Z_{2n-1})|Z_n,\dots,Z_i].$$ 
We now show that for all $i \in {n,\dots,2n-1}$,
	\begin{equation}
	\begin{aligned}\label{LIP2}&|f_{i}(z_1,\dots,z_{i-n+1})-f_{i}(\bar z_1,\dots,\bar z_{i-n+1})|\\
	&\le \frac{1}{n} \int_{\mathsf{Z}}\cdots\int_{\mathsf{Z}}\sup_{h\in\mathscr{H}}\Bigg(\sum_{j=1}^{i-n+1}|\mathcal{L}_{h}(z_j)-\mathcal{L}_{h}(\bar z_j)|\\
	&\hspace{2.7cm}+\sum_{j=1}^{2n-i-1}|\mathcal{L}_{h}(F^j(z_i,y_j))-\mathcal{L}_{h}(F^j(\bar z_i,y_j))|\Bigg)\Prob_{\vartheta_{2n-i-1}}(\D y_{2n-i-1})\cdots\Prob_{\vartheta_1}(\D y_1),
	\end{aligned}
	\end{equation}
	where $$F^1(z_i,y_1)=F(z_i,y_1)\qquad\text{and}\qquad F^j(z_i,y_j)=F(F^{j-1}(z_i,y_{j-1}),y_{j}).$$
	We prove this claim  by induction.
	For $i=2n-1$, from \cref{LIP1}, we have that \begin{align*}
	|f_{2n-1}(z_1,\dots,z_{n})-f_{2n-1}(\bar z_1,\dots,\bar z_{n})|&=|\varphi(z_1,\dots,z_{n})-\varphi(\bar z_1,\dots,\bar z_{n})|\\
	&\le \sup_{h\in\mathscr{H}}\frac{1}{n}\sum_{i=1}^{n}|\mathcal{L}_{h}(z_i)-\mathcal{L}_{h}(\bar z_i)|.
	\end{align*}
	Assuming  \cref{LIP2} holds for some $i\in\{n+1,\dots,2n-1\}$, we establish it for $i-1$. Using conditional expectation properties, we obtain \begin{align*}
	&|f_{i-1}(z_1,\dots,z_{i-n})-f_{i-1}(\bar z_1,\dots,\bar z_{i-n})|\\
	&=|\mathbb{E}^\upmu[f_{i}(Z_n,\dots,Z_{i})|Z_n=z_1,\dots,Z_{i-1}=z_{i-n}]-\mathbb{E}^\upmu[f_{i}(Z_n,\dots,Z_{i})|Z_n=\bar z_1,\dots,Z_{i-1}=\bar z_{i-n}]|\\
	&\le\int_{\mathsf{Z}}|f_{i}(z_1,\dots,z_{i-n},F(z_{i-n},y))-f_{i}(\bar z_1,\dots,\bar z_{i-n},F(\bar z_{i-n},y))|\Prob_{\vartheta_1}(\D y)\\
	&\le \frac{1}{n} \int_{\mathsf{Z}}\cdots\int_{\mathsf{Z}}\sup_{h\in\mathscr{H}}\Bigg(\sum_{j=1}^{i-n}|\mathcal{L}_{h}(z_j)-\mathcal{L}_{h}(\bar z_j)|\\
	&\hspace{4cm}+\sum_{j=1}^{2n-i}|\mathcal{L}_{h}(F^j(z_{i-n},y_j))-\mathcal{L}_{h}(F^j(\bar z_{i-n},y_j))|\Bigg)\Prob_{\vartheta_{2n-i}}(\D y_{2n-i})\cdots\Prob_{\vartheta_1}(\D y_1),
	\end{align*}
	which completes the proof of \cref{LIP2}.
	Next, define \begin{align*}V_{i}&:=f_{i}(Z_n,\dots,Z_i)-f_{i-1}(Z_n,\dots,Z_{i-1})\\ L_{i}&:=\inf_{z_i\in\mathsf{Z}}f_{i}(Z_n,\dots,Z_{i-1},z_i)-f_{i-1}(Z_n,\dots,Z_{i-1})\\ R_{i}&:=\sup_{z_i\in\mathsf{z}}f_{i}(Z_n,\dots,Z_{i-1},z_i)-f_{i-1}(Z_n,\dots,Z_{i-1}),\end{align*}
	with
	\begin{align*}
	V_{n}&:=f_{n}(Z_n)-\mathbb{E}^\upmu[\varphi(Z_n,\dots,Z_{2n-1})]\\
	L_{n}&:=\inf_{z_1\in\mathsf{Z}}f_{n}(z_1)-\mathbb{E}^\upmu[\varphi(Z_n,\dots,Z_{2n-1})]\\ R_{n}&:=\sup_{z_1\in\mathsf{Z}}f_{n}(z_1)-\mathbb{E}^\upmu[\varphi(Z_n,\dots,Z_{2n-1})].
	\end{align*}
	It follows that, $L_{i}\le V_{i}\le R_{i},$ $\mathbb{E}^\upmu[V_{i}|Z_n,\dots,Z_{i-1}]=0$, $\mathbb{E}^\upmu[V_{n}]=0$,  $$\sum_{i=n}^{2n-1}V_{i}=\varphi(Z_n,\dots,Z_{2n-1})-\mathbb{E}^\upmu[\varphi(Z_n,\dots,Z_{2n-1})]$$ 
	and \begin{align*} R_{i}-L_{i}&=\sup_{z_{i}\in\mathsf{Z}}f_{i}(Z_n,\dots,Z_{i-1},z_{i})-\inf_{\bar z_{i}\in\mathsf{Z}}f_{i}(Z_n,\dots,Z_{i-1},z_{i})\\
	&=\sup_{z_{i},\bar z_{i}\in\mathsf{Z}}(f_{i}(Z_n,\dots,Z_{i-1},z_{i})-f_{i}(Z_n,\dots,Z_{i-1},\bar z_{i}))\\
	&\le\frac{1}{n}\sup_{z_i,\bar z_i\in\mathsf{Z}}  \int_{\mathsf{Z}}\cdots\int_{\mathsf{Z}}\sup_{h\in\mathscr{H}}\Bigg(|\mathcal{L}_{h}(z_i)-\mathcal{L}_{h}(\bar z_i)|\\
	&\hspace{2.7cm}+\sum_{j=1}^{2n-i-1}|\mathcal{L}_{h}(F^j(z_i,y_j))-\mathcal{L}_{h}(F^j(\bar z_i,y_j))|\Bigg)\Prob_{\vartheta_{2n-i-1}}(\D y_{2n-i-1})\cdots\Prob_{\vartheta_1}(\D y_1)\\
	&\le \frac{\ell_\mathscr{H}(1+\ell_F+\cdots+\ell_F^{2n-i-1})}{n}\sup_{z_{i},\bar z_{i}\in\mathsf{Z}}\mathrm{d}_\mathsf{Z}(z,\bar z)\\
	&\le\frac{\ell_\mathscr{H}/(1-\ell_F)}{n} .\end{align*}
	In particular,  $L_{i}\le V_{i}\le L_i+(\ell_\mathscr{H}/(1-\ell_F))/n.$ For all $s>0$, we now have
	\begin{align*}
	&\mathbb{P}^\upmu\bigl(\varphi(Z_n,\dots,Z_{2n-1})-\mathbb{E}^\upmu[\varphi(Z_n,\dots,Z_{2n-1})]\ge \varepsilon\bigr)\\
	&=\mathbb{P}^\upmu\left(\sum_{i=n}^{2n-1}V_{i}\ge \varepsilon\right)\\
	&=\mathbb{P}^\upmu\bigl(\E^{s\sum_{i=n}^{2n-1}V_{i}}\ge \E^{s\varepsilon}\bigr)\\
	&\le \E^{-s\varepsilon}\mathbb{E}^\upmu\bigl[\E^{s\sum_{i=n}^{2n-1}V_{i}}\bigr]\\
	&=\E^{-s\varepsilon}\mathbb{E}^\upmu\left[\bigl[\E^{s\sum_{i=n}^{2n-2}V_{i}}\mathbb{E}^\upmu\bigl[\E^{sV_{2n-1}}|Z_n,\dots,Z_{2n-2}\bigr]\right].
	\end{align*}
	Using a standard concentration argument (see \cite[Lemma D.1]{Mohri-Rostamizadeh-Talwalkar-Book-2018}), we obtain $$\mathbb{E}^\upmu\bigl[\E^{sV_{2n-1}}|Z_1,\dots,Z_{2n-2}\bigr]\le\E^{s^2(\ell_\mathscr{H}/(1-\ell_F))^2/8n^2}.$$
	Thus,   $$\mathbb{P}^\upmu(\varphi(Z_n,\dots,Z_{2n-1})-\mathbb{E}^\upmu[\varphi(Z_n,\dots,Z_{2n-1})]\ge \varepsilon)\le \E^{-s\varepsilon+s^2(\ell_\mathscr{H}/(1-\ell_F))^2/8n}.$$
	Minimizing the function $s\mapsto \E^{-s\varepsilon+s^2(\ell_\mathscr{H}/(1-\ell_F))^2/8n}$ we arrive at
	$$\Prob^\upmu\bigl(\varphi(Z_n,\dots,Z_{2n-1})-\mathbb{E}^\upmu[\varphi(Z_n,\dots,Z_{2n-1})]\ge \varepsilon\bigr)\le \E^{-2\varepsilon^2n/(\ell_\mathscr{H}/(1-\ell_F))^2},$$
	which completes the proof. 
\end{proof}

In the following lemma, we analyze the term $\mathbb{E}^\upmu[\varphi(Z_n,\dots,Z_{2n-1})]$.

\begin{lemma}\label{LEM2} Assume (A1) and (A2). For any $\upmu\in\mathscr{P}(\sfZ)$, the following inequality holds:
	$$\mathbb{E}^\upmu[\varphi(Z_n,\dots,Z_{2n-1})]\le2\mathcal{R}_{n,\uppi}(\mathscr{H})+\ell_\mathscr{H}\mathscr{W}(\upmu\mathcal{P}^n,\uppi).$$
\end{lemma}
\begin{proof}
	We start by rewriting $\mathbb{E}^\upmu[\varphi(Z_n,\dots,Z_{2n-1})]$ as follows:
	\begin{align*}&\mathbb{E}^\upmu[\varphi(Z_n,\dots,Z_{2n-1})]\\&=\mathbb{E}^\uppi[\varphi(Z_n,\dots,Z_{2n-1})]+
	\mathbb{E}^\upmu[\varphi(Z_n,\dots,Z_{2n-1})]-\mathbb{E}^\uppi[\varphi(Z_n,\dots,Z_{2n-1})]\\
	&=\mathbb{E}^\uppi[\varphi(Z_0,\dots,Z_{n-1})]+ \mathbb{E}^\upmu[\mathbb{E}^{Z_n}[\varphi(Z_0,\dots,Z_{n-1})]]-\mathbb{E}^\uppi[\varphi(Z_0,\dots,Z_{n-1})]\\
	&=\mathbb{E}^\uppi[\varphi(Z_0,\dots,Z_{n-1})]+ \int_{\mathsf{Z}}\mathbb{E}^{z}[\varphi(Z_0,\dots,Z_{n-1})](\upmu\mathcal{P}^n(\D z)-\uppi(\D z)).\end{align*} 
	Next, we establish that the function  $z\mapsto \mathbb{E}^z[\varphi(Z_0,\dots,Z_{n-1})]$ is Lipschitz with Lipschitz constant at most $\ell_\mathscr{H}$. For any $z,\bar z\in\mathsf{Z}$, using \cref{LIP1}, we obtain: \begin{align*}
	&|\mathbb{E}^z[\varphi(Z_0,\dots,Z_{n-1})]-\mathbb{E}^{\bar z}[\varphi(Z_0,\dots,Z_{n-1})]|\\
	&\le\mathbb{E}^{n-1}[|\varphi(z,F^1(z,\vartheta_1),\dots,F^{n-1}(z,\vartheta_{n-1}))-\varphi(\bar z,F^1(\bar z,\vartheta_1),\dots,F^{n-1}(\bar z,\vartheta_{n-1}))|]\\
	&\le\mathbb{E}^{n-1}\left[ \sup_{h\in\mathscr{H}}\frac{1}{n}\sum_{i=0}^{n-1}|\mathcal{L}_{h}(F^i(z,\vartheta_i))-\mathcal{L}_{h}(F^i(\bar z,\vartheta_i))|\right]\\
	&\le \frac{1}{n}\sum_{i=0}^{n-1}\ell_\mathscr{H}\ell_F^i{\rm d}_\mathsf{Z}(z,\bar z)\\
	&\le \ell_\mathscr{H}{\rm d}_\mathsf{Z}(z,\bar z),
	\end{align*}
	where $F^0(z,\vartheta_0):=z$.
	Applying (A1) and the comment preceding Theorem 1.7 in \cite{Sandric-2017}, we conclude that $$ 	\mathbb{E}^\upmu[\varphi(Z_n,\dots,Z_{2n-1})]\le\mathbb{E}^\uppi[\varphi(Z_0,\dots,Z_{n-1})]+\ell_\mathscr{H}\mathscr{W}(\upmu\mathcal{P}^n,\uppi).$$	 
	To bound	$\mathbb{E}^\uppi[\varphi(Z_0,\dots,Z_{n-1})],$ we observe that
	\begin{align*}
	\mathbb{E}^\uppi[\varphi(Z_0,\dots,Z_{n-1})]=\mathbb{E}^\uppi\bigl[\sup_{h\in\mathscr{H}}|\hat{\rm er}_{n}(h)- {\rm er}_\uppi(h)|\bigr]=\mathbb{E}^\uppi\bigl[\sup_{h\in\mathscr{H}\cup -\mathscr{H}}(\hat{\rm er}_{n}(h)- {\rm er}_\uppi(h))\bigr],\end{align*} 
	where $-\mathscr{H}=\{-h:h\in\mathscr{H}\}$. Using the approach  in \cite[the proof of Theorem 3.3]{Mohri-Rostamizadeh-Talwalkar-Book-2018}, we establish that 
	$$
	\mathbb{E}^\uppi[\varphi(Z_0,\dots,Z_{n-1})]\le2\mathcal{R}_{n,\uppi}(\mathscr{H}\cup -\mathscr{H})=2\mathcal{R}_{n,\uppi}(\mathscr{H}).$$ This completes the proof.
\end{proof}

\begin{remark}	
	It is worth noting that $$	\mathbb{E}^\uppi[\varphi(Z_0,\dots,Z_{n-1})]\ge \frac{1}{2}\mathcal{R}_{n,\uppi}(\mathscr{H})-L_\mathscr{H}\sqrt{\frac{\log 2}{2n}}.$$ In particular, $\lim_{n\to\infty}\mathbb{E}^\uppi[\varphi(Z_0,\dots,Z_{n-1})]=0$ if, and only if, $\lim_{n\to\infty}\mathcal{R}_{n,\uppi}(\mathscr{H})=0.$ Namely, let $\{\bar Z_n\}_{n\ge0}$ be an independent copy of $\{Z_n\}_{n\ge0}$. We  have  
	\begin{align*}
	\mathcal{R}_{n,\uppi}(\mathscr{H})&= \mathbb{E}^\uppi\times\mathbb{E}^\sigma\left[\sup_{h\in\mathscr{H}}\frac{1}{n}\sum_{i=0}^{n-1}\sigma_i\mathcal{L}_{h}(Z_i)\right]\\
	&\le \mathbb{E}^\uppi\times\mathbb{E}^\sigma\left[\sup_{h\in\mathscr{H}}\frac{1}{n}\sum_{i=0}^{n-1}\sigma_i(\mathcal{L}_{h}(Z_i)-{\rm er}_\uppi(h))\right]+\mathbb{E}^\sigma\left[\sup_{h\in\mathscr{H}}\frac{1}{n}\sum_{i=0}^{n-1}\sigma_i{\rm er}_\uppi(h)\right]\\
	&\le \mathbb{E}^\uppi\times\bar{\mathbb{E}}^\uppi\times\mathbb{E}^\sigma\left[\sup_{h\in\mathscr{H}}\frac{1}{n}\sum_{i=0}^{n-1}\sigma_i(\mathcal{L}_{h}(Z_i)-\mathcal{L}_{h}(\bar Z_i))\right]+\frac{L_\mathscr{H}}{n}\mathbb{E}^\sigma\left[\left|\sum_{i=0}^{n-1}\sigma_i\right|\right]\\
	&\le \mathbb{E}^\uppi\times\mathbb{E}^\sigma\left[\sup_{h\in\mathscr{H}}\frac{1}{n}\sum_{i=0}^{n-1}\sigma_i(\mathcal{L}_{h}(Z_i)-{\rm er}_\uppi(h))\right]\\
	&\ \ \ \ +\bar{\mathbb{E}}^\uppi\times\mathbb{E}^\sigma\left[\sup_{h\in\mathscr{H}}\frac{1}{n}\sum_{i=0}^{n-1}\sigma_i({\rm er}_\uppi(h)-\mathcal{L}_{h}(\bar Z_i))\right] +\frac{L_\mathscr{H}}{n}\mathbb{E}^\sigma\left[\sup_{a\in\{(-1,\dots,-1),(1,\dots,1)\}}\langle \sigma,a\rangle\right]\\
	&\le 2 \mathbb{E}^\uppi[\varphi(Z_0,\dots,Z_{n-1})]+\frac{L_\mathscr{H}}{n}\mathbb{E}^\sigma\left[\sup_{a\in\{(-1,\dots,-1),(1,\dots,1)\}}\langle \sigma,a\rangle\right],
	\end{align*}
	where $\sigma=(\sigma_0,\dots,\sigma_{n-1})$ and $\langle\cdot,\cdot \rangle$ stands for the standard scalar product in $\R^n$.
Applying \cite[Theorem 3.7]{Mohri-Rostamizadeh-Talwalkar-Book-2018}, we conclude that $$\mathbb{E}^\sigma\left[\sup_{a\in\{(-1,\dots,-1),(1,\dots,1)\}}\langle \sigma,a\rangle\right]\le \sqrt{2 n\log 2},$$ which proves the claim. \qed
\end{remark}

\bigskip

From \Cref{LEM1,LEM2}, we conclude that
for any $\upmu\in\mathscr{P}(\sfZ)$ and  $\varepsilon\in(0,1)$, the following holds: \begin{align*}\Prob^\upmu\left(\sup_{h\in\mathscr{H}}|\hat{\rm er}_{n,2n}(h)- {\rm er}_\uppi(h)|\le 2\mathcal{R}_{n,\uppi}(\mathscr{H})+\ell_\mathscr{H}\mathscr{W}(\upmu\mathcal{P}^n,\uppi)+\varepsilon\right)\ge 1-\E^{-2\varepsilon^2n/(\ell_\mathscr{H}/(1-\ell_F))^2}.\end{align*} Equivalently, we have \begin{equation}\begin{aligned}\label{SUP}&\Prob^\upmu\left(\sup_{h\in\mathscr{H}}|\hat{\rm er}_{n,2n}(h)- {\rm er}_\uppi(h)|\le 2\mathcal{R}_{n,\uppi}(\mathscr{H})+\ell_\mathscr{H}\ell_F^n\mathscr{W}(\upmu,\uppi)+\varepsilon\right)\\&\ge 1-\E^{-2\varepsilon^2n/(\ell_\mathscr{H}/(1-\ell_F))^2}.\end{aligned}\end{equation}

In the following lemma, we derive an analogous result in terms of the $n$-empirical Rademacher complexity.

\begin{lemma}\label{LEM3} Assume (A1) and (A2).
	Then, for  any $\varepsilon\in(0,1)$, we have
	\begin{align*}\Prob^\uppi\left(\sup_{h\in\mathscr{H}}|\hat{\rm er}_{n,2n}(h)- {\rm er}_\uppi(h)|\le 2\hat{\mathcal{R}}_{n,(Z_0,\dots,Z_{n-1})}(\mathscr{H})+3\varepsilon\right)\ge 1-\E^{-2\varepsilon^2n/(\ell_\mathscr{H}/(1-\ell_F))^2}.\end{align*}
\end{lemma}
\begin{proof}
	From \cref{SUP}, for any $\delta\in(0,1)$, we obtain \begin{align}\label{eq:rad}\Prob^\uppi\Bigg(\sup_{h\in\mathscr{H}}|\hat{\rm er}_{n,2n}(h)- {\rm er}_\uppi(h)|\le 2\mathcal{R}_{n,\uppi}(\mathscr{H})+(\ell_\mathscr{H}/(1-\ell_F))\sqrt{\frac{\log(2/\delta)}{2n}}\Bigg)\ge 1-\delta/2.\end{align}
	Next, let $\phi:\mathsf{Z}^n\to\R$ be defined as $\phi(z)\df-\hat{\mathcal{R}}_{n,z}(\mathscr{H}).$ 
	According to assumption (A2),  we have $|\phi(z)|\le\ell_\mathscr{H}$ for all $z\in\mathsf{Z}$.
Consequently, for each $i = n, \dots , 2n-1$, the conditional expectation  $\mathbb{E}^\uppi[\phi(Z_n,\dots,Z_{2n-1})|Z_{n},\dots,Z_i]$ is well defined, and 
 $\mathbb{E}^\uppi[\phi(Z_0,\dots,Z_{n-1})]=-\mathcal{R}_{n,\uppi}(\mathscr{H}).$    Furthermore, for any $z,\bar{z}\in\mathsf{Z}^n$, we obtain
\begin{align*}
|\phi(z_1,\dots,z_n)-\phi(\bar z_1,\dots,\bar z_n)|&=\frac{1}{n}\left|\mathbb{E}^\sigma\bigl[\sup_{h\in\mathscr{H}}\sum_{i=1}^n\sigma_i\mathcal{L}_h(z_i)-\sup_{h\in\mathscr{H}}\sum_{i=1}^n\sigma_i\mathcal{L}_h(\bar z_i)\bigr]\right|\\
&\le \frac{1}{n}\sup_{h\in\mathscr{H}}\sum_{i=1}^n|\mathcal{L}_h(z_i)-\mathcal{L}_h(\bar z_i)|.
\end{align*} 
Define the function $f_{i}:\mathsf{Z}^{i-n+1}\to\R$ by $$f_{i}(Z_n,\dots,Z_i)\df\mathbb{E}^\upmu[\phi(Z_n,\dots,Z_{2n-1})|Z_n,\dots,Z_i].$$ 
Proceeding analogously to \Cref{LEM1}, we obtain 
\begin{equation*}
\begin{aligned}\label{LIP}&|f_{i}(z_1,\dots,z_{i-n+1})-f_{i}(\bar z_1,\dots,\bar z_{i-n+1})|\\
&\le \frac{1}{n} \int_{\mathsf{Z}}\cdots\int_{\mathsf{Z}}\sup_{h\in\mathscr{H}}\Bigg(\sum_{j=1}^{i-n+1}|\mathcal{L}_{h}(z_j)-\mathcal{L}_{h}(\bar z_j)|\\
&\hspace{2.7cm}+\sum_{j=1}^{2n-i-1}|\mathcal{L}_{h}(F^j(z_i,y_j))-\mathcal{L}_{h}(F^j(\bar z_i,y_j))|\Bigg)\Prob_{\vartheta_{2n-i-1}}(\D y_{2n-i-1})\cdots\Prob_{\vartheta_1}(\D y_1),
\end{aligned}
\end{equation*}
where $$F^1(z_i,y_1)=F(z_i,y_1)\qquad\text{and}\qquad F^j(z_i,y_j)=F(F^{j-1}(z_i,y_{j-1}),y_{j}).$$	
Next, define  \begin{align*}V_{i}&\df f_{i}(Z_n,\dots,Z_i)-f_{i-1}(Z_n,\dots,Z_{i-1})\\ L_{i}&\df \inf_{z_i\in\mathsf{Z}}f_{i}(Z_n,\dots,Z_{i-1},z_i)-f_{i-1}(Z_n,\dots,Z_{i-1})\\ R_{i}&\df \sup_{z_i\in\mathsf{z}}f_{i}(Z_n,\dots,Z_{i-1},z_i)-f_{i-1}(Z_n,\dots,Z_{i-1}),\end{align*}
	with
	\begin{align*}
	V_{n}&:=f_{n}(Z_n)-\mathbb{E}^\uppi[\phi(Z_n,\dots,Z_{2n-1})]\\
	L_{n}&:=\inf_{z_1\in\mathsf{Z}}f_{n}(z_1)-\mathbb{E}^\uppi[\phi(Z_n,\dots,Z_{2n-1})]\\ R_{n}&:=\sup_{z_1\in\mathsf{Z}}f_{n}(z_1)-\mathbb{E}^\uppi[\phi(Z_n,\dots,Z_{2n-1})].
	\end{align*}
	It follows that, as in \Cref{LEM1}, 
 $L_{i}\le V_{i}\le R_{i},$  $\mathbb{E}^\uppi[V_{i}|Z_n,\dots,Z_{i-1}]=0$, $\mathbb{E}^\uppi[V_{n}]=0$,  $$\sum_{i=n}^{2n-1}V_{i}=\phi(Z_n,\dots,Z_{2n-1})-\mathbb{E}^\uppi[\phi(Z_n,\dots,Z_{2n-1})],$$ 
 $$ R_{i}-L_{i}\le\frac{\ell_\mathscr{H}/(1-\ell_F)}{n} ,$$ (i.e.,
	  $L_{i}\le V_{i}\le L_i+(\ell_\mathscr{H}/(1-\ell_F))/n$), and 
	\begin{align*}
	\Prob^\uppi\bigl(\mathcal{R}_{n,\uppi}(\mathscr{H})-\hat{\mathcal{R}}_{n,(Z_0,\dots,Z_{n-1})}(\mathscr{H})\ge \varepsilon\bigr)
	&=\mathbb{P}^\uppi\bigl(\phi(Z_n,\dots,Z_{2n-1})-\mathbb{E}^\uppi[\phi(Z_n,\dots,Z_{2n-1})]\ge \varepsilon\bigr)\\
	&=\mathbb{P}^\uppi\left(\sum_{i=n}^{2n-1}V_{i}\ge \varepsilon\right)\\
	&\le\E^{-2\varepsilon^2n/(\ell_\mathscr{H}/(1-\ell_F))^2},\end{align*}
 i.e.,
	
	$$\Prob^\uppi\left(\mathcal{R}_{n,\uppi}(\mathscr{H})\le \hat{\mathcal{R}}_{n,(Z_0,\dots,Z_{n-1})}(\mathscr{H})+ (\ell_\mathscr{H}/(1-\ell_F))\sqrt{\frac{\log(2/\delta)}{2n}}\right)\ge 1-\delta/2.$$ Combining this with \cref{eq:rad} yields the desired result.
	\end{proof}

Finally, we prove \Cref{TM}.
\begin{proof}[Proof of \Cref{TM}]
	Fix $\varepsilon\in(0,1)$, and	let $\mathcal{A}^\varepsilon$ be the $\varepsilon$-ERM algorithm for $\mathscr{H}$. From \cref{SUP}, it follows that, with probability at least $1-2\E^{-2\varepsilon^2n/(\ell_\mathscr{H}/(1-\ell_F))^2}$, we have
that
	\begin{align*}\mathrm{er}_\uppi(\mathcal{A}^\varepsilon(Z_0,\dots,Z_{n-1}))&\le\hat{\mathrm{er}}_{n,2n}(\mathcal{A}^\varepsilon(Z_0,\dots,Z_{n-1}))+2\mathcal{R}_{n,\uppi}(\mathscr{H})+\ell_\mathscr{H}\ell_F^n\mathscr{W}(\upmu,\uppi)+\varepsilon\\
	&\le\inf_{h\in\mathscr{H}}\hat{\mathrm{er}}_{n,2n}(h)+2\mathcal{R}_{n,\uppi}(\mathscr{H})+\ell_\mathscr{H}\ell_F^n\mathscr{W}(\upmu,\uppi)+2\varepsilon\\
	&\le \hat{\mathrm{er}}_{n,2n}(\bar h)+2\mathcal{R}_{n,\uppi}(\mathscr{H})+\ell_\mathscr{H}\ell_F^n\mathscr{W}(\upmu,\uppi)+2\varepsilon\\
	&\le \mathrm{er}_\uppi(\bar h)+4\mathcal{R}_{n,\uppi}(\mathscr{H})+2\ell_\mathscr{H}\ell_F^n\mathscr{W}(\upmu,\uppi)+3\varepsilon\\
	&\le \mathrm{opt}_\uppi(\mathscr{H})+4\mathcal{R}_{n,\uppi}(\mathscr{H})+2\ell_\mathscr{H}\ell_F^n\mathscr{W}(\upmu,\uppi)+4\varepsilon,\end{align*}
	where   $\bar h\in \mathscr{H}$satisfies $\mathrm{er}_\uppi(\bar h)\le\mathrm{opt}_\uppi(\mathscr{H})+\varepsilon$.
	The second claim follows analogously by applying \Cref{LEM3}.
\end{proof}

\begin{remark}
	Recall that $$\sup_{z_1,\dots,z_n\in\sfZ}\varphi(z_1,\dots,z_n)\le\ell_\mathscr{H}.$$  
Consequently, we obtain
	$$\int_{\mathsf{Z}}\mathbb{E}^{z}[\varphi(Z_0,\dots,Z_{n-1})](\upmu\mathcal{P}^n(\D z)-\uppi(\D z))\le \ell_\mathscr{H} \lVert\upmu\mathcal{P}^n-\uppi\rVert_{{\rm TV}},$$
	where $$\lVert\upeta\rVert_{{\rm TV}}:=\frac{1}{2}\sup_{\substack{f:\sfZ\to\R,\ \lVert f\lVert_\infty\le 1}}|\upeta(f)|$$ denotes the total variation norm of a  signed measure $\upeta(\D z)$ on $\sfZ$.
Using this, the conclusion of \Cref{LEM2} can be expressed as
	$$\mathbb{E}^\upmu[\varphi(Z_n,\dots,Z_{2n-1})]\le2\mathcal{R}_{n,\uppi}(\mathscr{H})+\ell_\mathscr{H}
	\lVert\upmu\mathcal{P}^n-\uppi\rVert_{{\rm TV}}.$$ 
	Combining this with \Cref{LEM1}, we obtain an alternative formulation of the first assertion in \Cref{TM}: \begin{align*}&\Prob^\upmu\left(|\mathrm{er}_\uppi\bigl(\mathcal{A}^\varepsilon(Z_0,\dots,Z_{n-1})\bigr)-\mathrm{opt}_\uppi(\mathscr{H})|<4\mathcal{R}_{n,\uppi}(\mathscr{H})+2\ell_\mathscr{H}\lVert\upmu\mathcal{P}^n-\uppi\rVert_{{\rm TV}}+4\varepsilon\right)\\&\ge1-2\E^{-2\varepsilon^2n/(\ell_\mathscr{H}/(1-\ell_F))^2}.\end{align*}
	However, by \cite[Theorem 6.15]{Villani-Book-2009}, the Wasserstein distance satisfies $$\mathscr{W}(\upmu,\upnu)\le\lVert\upmu-\upnu\rVert_{{\rm TV}}\qquad\forall\upmu,\upnu\in\mathscr{P}(\sfZ).$$
In general, to ensure that $\lVert\upmu\mathcal{P}^n-\uppi\rVert_{{\rm TV}}$ converges to zero as $n\to \infty$,   it is necessary that $\{Z_n\}_{n\ge0}$ is irreducible and aperiodic (see, e.g., \cite{Meyn-Tweedie-Book-2009}). Thus, in this context, it is more natural and general to present the results in terms of the Wasserstein distance, as illustrated in \Cref{EX3}.
	\end{remark}



\section*{Acknowledgements} The author is grateful for the insightful and constructive comments received from two anonymous referees and the associate editor, which have led to improvements in the article.
 Financial support through   \textit{Croatian Science Foundation} under project 2277 is gratefully acknowledged.

\section*{Conflict of interest}
The authors declare that they have no conflict of interest.

\bibliographystyle{plain}
\bibliography{References}

\end{document}